\documentclass[11pt,a4paper]{article}

\usepackage[T1]{fontenc}
\usepackage[utf8]{inputenc}
\usepackage{mathptmx}
\usepackage[scaled=0.92]{helvet}
\usepackage{microtype}
\usepackage[a4paper,margin=1in]{geometry}
\usepackage{amsmath,amsthm}
\usepackage{graphicx}
\usepackage{xcolor}
\usepackage{xurl}
\usepackage{enumitem}
\usepackage{algorithm}
\usepackage{algorithmic}
\usepackage{caption}
\usepackage{titlesec}
\usepackage{fancyhdr}
\usepackage{tikz}
\usetikzlibrary{positioning,arrows.meta,calc}
\usepackage[numbers,sort&compress]{natbib}
\usepackage{hyperref}

\graphicspath{{assets/}}

\newcommand{\agmissingfigure}[1]{%
    \fbox{%
        \begin{minipage}[c][0.24\textheight][c]{0.88\linewidth}
            \centering\sffamily\small
            \textcolor{AGMuted}{Missing figure asset}\\[0.4em]
            \texttt{\detokenize{#1}}
        \end{minipage}%
    }%
}
\newcommand{\agincludegraphics}[2][]{%
    \IfFileExists{#2}{\includegraphics[#1]{#2}}{%
        \IfFileExists{assets/#2}{\includegraphics[#1]{#2}}{\agmissingfigure{#2}}%
    }%
}

\definecolor{AGBlue}{HTML}{1647F5}
\definecolor{AGInk}{HTML}{101318}
\definecolor{AGMuted}{HTML}{5F6673}
\definecolor{AGCyan}{HTML}{04A7C7}
\definecolor{AGRule}{HTML}{DDE4F0}
\definecolor{AGPale}{HTML}{F5F8FF}

\hypersetup{
    colorlinks=true,
    linkcolor=AGBlue,
    citecolor=AGBlue,
    urlcolor=AGCyan,
    pdftitle={Automatic Stability and Recovery for Neural Network Training},
    pdfauthor={Barak Or},
    pdfsubject={Preprint / Research manuscript. Not yet peer-reviewed.}
}

\titleformat{\section}
    {\large\sffamily\bfseries\color{AGBlue}}
    {\thesection}{0.75em}{}
    [\vspace{-0.25em}{\color{AGRule}\titlerule[0.7pt]}]
\titleformat{\subsection}
    {\normalsize\sffamily\bfseries\color{AGInk}}
    {\thesubsection}{0.65em}{}
\titleformat{\subsubsection}
    {\normalsize\sffamily\itshape\color{AGMuted}}
    {\thesubsubsection}{0.65em}{}
\titlespacing*{\section}{0pt}{2.3ex plus 0.6ex minus 0.2ex}{1.1ex}
\titlespacing*{\subsection}{0pt}{1.7ex plus 0.4ex minus 0.2ex}{0.7ex}

\setlist[itemize]{leftmargin=*,topsep=3pt,itemsep=2pt}
\setlist[enumerate]{leftmargin=*,topsep=3pt,itemsep=2pt}

\theoremstyle{plain}
\newtheorem{theorem}{Theorem}[section]
\newtheorem{proposition}[theorem]{Proposition}

\theoremstyle{definition}
\newtheorem{definition}[theorem]{Definition}
\newtheorem{assumption}[theorem]{Assumption}
\theoremstyle{remark}

\newcommand{\sep}{\unskip;\space}
\newenvironment{keyword}{%
    \par\vspace{0.6em}
    \noindent{\sffamily\bfseries\color{AGBlue}Keywords.}\space
}{\par\vspace{1.0em}}

\renewcommand{\headrulewidth}{0.6pt}
\renewcommand{\headrule}{\hbox to\headwidth{\color{AGRule}\leaders\hrule height \headrulewidth\hfill}}
\makeatletter
\renewcommand{\maketitle}{%
  \bgroup\setlength{\parindent}{0pt}%
  \begin{flushleft}
    \vspace*{-1.1cm}
    \agincludegraphics[height=0.86cm]{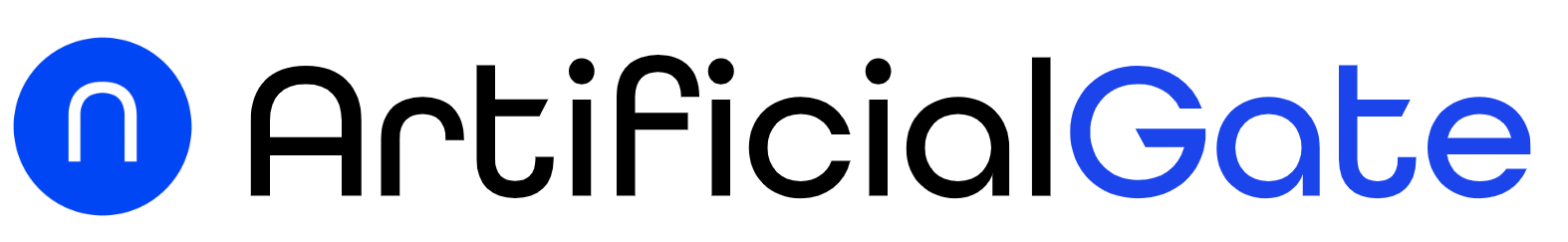}\\[0.35cm]
    {\color{AGBlue}\rule{\textwidth}{1.2pt}}\\[0.75cm]
    {\small\sffamily\bfseries\color{AGMuted}ArtificialGate Ltd. Research Manuscript\par}
    {\small\sffamily\color{AGMuted}Preprint / Research manuscript. Not yet peer-reviewed.\par}
    \vspace{0.35cm}
    {\LARGE\sffamily\bfseries\color{AGInk}\@title\par}
    \vspace{0.55cm}
    {\large\@author\par}
  \end{flushleft}
  \egroup
  \vspace{0.4cm}
  \thispagestyle{empty}%
}
\newcommand\blfootnote[1]{%
  \begingroup
  \renewcommand\thefootnote{}\footnote{#1}%
  \addtocounter{footnote}{-1}%
  \endgroup
}
\makeatother

\title{Automatic Stability and Recovery for Neural Network Training}
\author{%
  \textbf{Barak Or}\\
  \small ArtificialGate Ltd., Israel\\
  \texttt{barak@artificialgate.ai}
}
\date{}

\begin{document}
\maketitle

\blfootnote{This manuscript is formatted as an ArtificialGate Ltd. research manuscript. It is a preprint / research manuscript and has not yet been peer reviewed.}

\begin{abstract}
Training modern neural networks is increasingly fragile, with rare but severe destabilizing updates often causing irreversible divergence or silent performance degradation. Existing optimization methods primarily rely on preventive mechanisms embedded within the optimizer, offering limited ability to detect and recover from instability once it occurs.

We introduce a supervisory runtime stability framework that treats optimization as a controlled stochastic process. By isolating an innovation signal derived from secondary measurements, such as validation probes, the framework enables automatic detection and recovery from destabilizing updates without modifying the underlying optimizer. We provide theoretical runtime safety guarantees that formalize bounded degradation and recovery. Our implementation incurs minimal overhead and is compatible with memory-constrained training settings.
\end{abstract}


\section{Introduction}

Despite significant advances in optimization algorithms, model architectures, and large-scale infrastructure, practitioners routinely encounter training runs that diverge, collapse, or silently degrade after a small number of destabilizing updates. In such settings, a single pathological update-caused by an outlier minibatch, transient numerical instability, or abrupt distributional shift-can irreversibly corrupt the training trajectory, leading to wasted computation and unreliable experimental outcomes.

Most existing approaches to mitigating training instability focus on \emph{preventive mechanisms} embedded within the optimizer itself. Examples include adaptive learning rates and momentum schemes \cite{polyak1964some,kingma2014adam}, gradient clipping \cite{pascanu2013difficulty}, trust-region or curvature-aware methods \cite{martens2010deep}, and related variance-reduction techniques. While effective in many scenarios, these methods are inherently local and proactive: they aim to reduce the likelihood of instability, but provide limited protection once a destabilizing update has already occurred. In practice, when training collapses due to an unexpected event, standard pipelines offer no principled mechanism for detection, rollback, or recovery.

In this work, we propose a \emph{runtime stability layer} for neural network training that treats optimization as a controlled stochastic process. Rather than modifying the optimizer update rule, our approach operates above the optimizer, monitoring optimizer-proposed updates using secondary measurement signals that are not directly optimized by the training objective. Examples of such signals include small validation probes, loss trajectory statistics, or measures of gradient agreement. We refer to this discrepancy as an innovation signal: a runtime residual between the behavior expected under recent stable training dynamics and the behavior actually observed after an optimizer-proposed update. Large positive innovations indicate that the proposed update is inconsistent with the recent training regime and may therefore correspond to a destabilizing event.

When instability is detected, the controller intervenes through a conservative corrective action: rejecting the proposed update and restoring the training state to a previously accepted snapshot. The contributions of this work are as follows:
\begin{itemize}
    \item We introduce an optimizer-agnostic control-theoretic framework for runtime stability and recovery in neural network training.

    \item We provide theoretical guarantees in the form of runtime safety invariants, showing that the controller enforces bounded degradation with respect to the chosen measurement signal and enables recovery from destabilizing updates.
    \item We demonstrate empirically that the proposed method reduces training divergence, improves robustness under controlled destabilization, and recovers from catastrophic updates across two representative settings: CIFAR-10 image classification with ResNet-18 and character-level sequence modeling with a Transformer.

\end{itemize}

\section{Related Work}

Our work intersects several lines of research in optimization, training stability, and control-inspired learning. 

\subsection{Optimization and Adaptive Training Methods}

Prior work on training stability has primarily addressed instability by altering the optimization dynamics themselves. Momentum-based methods and adaptive optimizers, including SGD with momentum, Adam, and AdamW \cite{polyak1964some,kingma2014adam,loshchilov2017decoupled}, rescale or smooth parameter updates using local gradient statistics. Complementary mechanisms such as gradient clipping \cite{pascanu2013difficulty} and trust-region or curvature-aware methods \cite{martens2010deep} constrain update magnitudes or local geometry in order to reduce the likelihood of divergence.

More recent methods, such as Sharpness-Aware Minimization (SAM) \cite{foret2020sharpness} and the Lookahead optimizer \cite{zhang2019lookahead}, seek to improve generalization or smooth training dynamics by modifying the update rule or maintaining auxiliary parameter trajectories. While effective in many settings, these approaches are inherently \emph{preventive}: they alter how updates are computed in order to reduce the likelihood of instability. Once a destabilizing update has already been applied, however, standard optimizers provide no mechanism for detecting the failure or recovering from it.

\subsection{Kalman-Inspired and Bayesian Training Approaches}

Several prior works draw inspiration from Kalman filtering or Bayesian inference to improve neural network training. Early work explored Kalman-inspired updates for neural network training \cite{singhal1989training, haykin2004kalman}. More recent methods incorporate uncertainty estimates or Kalman-inspired updates to denoise gradients, adapt learning rates, or approximate curvature information \cite{ritter2018online, osawa2019practical}. Unlike Bayesian or Kalman-style optimizers, we use innovation signals solely for runtime update acceptance, without maintaining probabilistic state or modifying the optimizer.

\subsection{Robustness, Trust Regions, and Training Stability}

Another related line of work focuses on robustness through trust-region methods, constrained optimization, or smoothing techniques such as exponential moving averages of parameters. Trust-region and second-order methods explicitly constrain update magnitudes to prevent divergence \cite{martens2010deep}, while parameter averaging techniques improve stability by smoothing the optimization trajectory \cite{izmailov2018averaging}.

While conceptually related, such techniques operate at the level of update computation and are primarily preventive. They do not explicitly monitor training outcomes using secondary signals, nor do they provide mechanisms for rollback or recovery once an update has been applied. Early stopping and checkpointing offer coarse-grained protection but require manual intervention and do not integrate into the training loop as automated control mechanisms.

In contrast, our method introduces a fine-grained, automated stability controller that continuously monitors training behavior at runtime and intervenes when instability is detected. This enables rapid recovery from rare but severe failures while preserving standard training behavior during stable regimes. Simple heuristics such as gradient clipping are inherently local and preventive, and cannot recover a model once a destabilizing update has corrupted the parameters \cite{pascanu2013difficulty}. In contrast, the proposed approach enables explicit rollback based on secondary signals external to the optimizer.

\section{Models and Training Setup}

We evaluate the proposed runtime stability controller on representative neural network architectures drawn from two widely used model classes: convolutional networks for vision and Transformer-based models for sequence modeling. All models are trained using standard optimization pipelines, without architectural modifications or task-specific tuning beyond conventional choices.

\subsection{Vision Model}

For image classification, we use a ResNet-18 architecture \cite{he2016deep} with a standard classification head. Training is performed on the CIFAR-10 dataset \cite{krizhevsky2009learning} using cross-entropy loss.

Optimization is carried out using AdamW \cite{loshchilov2017decoupled}. Mini-batches of size 128 are used throughout training. Rather than training for a fixed number of epochs, we adopt a fixed-step schedule to enable precise alignment of destabilizing perturbations and controller interventions across runs.

\subsection{Sequence Model}

For sequence modeling, we employ a character-level Transformer encoder \cite{vaswani2017attention}. A linear output head maps hidden representations to vocabulary logits.

The task is next-character prediction on a controlled synthetic text corpus. Training uses cross-entropy loss over the predicted next character and is performed using AdamW \cite{loshchilov2017decoupled}. This setting provides a lightweight sequence-modeling test case for evaluating whether the proposed runtime stability controller generalizes beyond convolutional architectures.

\subsection{Training Protocol}

All models are trained for a fixed horizon of 250 optimization steps. Baseline and controlled runs share identical initialization and mini-batch order.

To evaluate recovery behavior, we introduce a controlled destabilization window by amplifying gradient magnitudes for several consecutive updates. Outside this window, training proceeds without modification. The runtime stability controller monitors training behavior using a fixed probe set and intervenes only when instability is detected. All experiments are repeated over $N=20$ independent random seeds, and results are reported as mean and variance across runs.

To evaluate recovery behavior under a reproducible stress test, we introduce a controlled destabilization window by amplifying gradient magnitudes for several consecutive updates. This perturbation is intentionally severe and is not meant to model the full distribution of naturally occurring training instabilities. Rather, it provides a diagnostic setting in which the effect of a destabilizing update is clearly observable, allowing us to isolate whether the controller can detect behavior inconsistent with recent stable dynamics, reject the proposed update, and restore a previously accepted training state. Unless otherwise stated, the perturbation is applied at step $t=120$ using gradient scale $\zeta = 300$ for a window of 10 iterations.

\section{Training as a Controlled Stochastic System}
We model neural network training as a controlled stochastic dynamical process. Let $\theta_t \in \mathrm{R}^d$ denote the model parameters at iteration $t$. A standard optimizer computes gradients of the training loss and proposes an update $\Delta \theta_t$, yielding a candidate next state
\begin{equation}
\theta_t^{\mathrm{prop}} = \theta_t + \Delta \theta_t .
\end{equation}
We interpret this proposal as a prediction of the next training state.

To assess whether the proposed update is consistent with stable training behavior, we introduce a \emph{measurement signal} that evaluates the effect of the update using information not directly optimized by the training objective. Such signals may include validation probe loss, statistics of the training loss trajectory, or measures of agreement between gradients computed on different minibatches. Importantly, these measurements are not assumed to be unbiased estimates of the training loss, nor do they require strong smoothness or stationarity assumptions.

The controller compares the observed measurement outcome associated with a proposed update against an expected behavior under stable training dynamics. The resulting deviation, referred to as an \emph{innovation signal}, serves as an indicator of whether the proposed update is consistent with recent training behavior. A precise definition of the innovation signal and its estimation is given in Section~\ref{sec:controller}.

In the implementation considered in this work, the innovation signal governs a binary runtime decision: the proposed update is either accepted and applied, or rejected via rollback to a previously accepted training state. The controller evaluates updates before acceptance and intervenes only when instability is detected. These interventions operate at the level of training execution rather than update computation.

\section{Runtime Stability Controller}
\label{sec:controller}

We now describe the proposed runtime stability controller. The controller operates as an external layer in the training loop, evaluating optimizer-proposed updates and intervening only when instability is detected. Crucially, the controller does not modify the optimizer update rule; instead, it governs the acceptance and execution of updates based on observed training behavior.
\subsection{Innovation Signal as a Runtime Consistency Check}

The core mechanism underlying the proposed controller is the use of an \emph{innovation signal} as a runtime consistency check for optimizer-proposed updates. Importantly, we do not interpret the innovation probabilistically, nor do we assume a generative or state-space model of training dynamics. Instead, the innovation serves as a pragmatic indicator of whether a proposed update is consistent with recently observed, stable training behavior.

\begin{definition}[Admissible Innovation Signal]
An innovation signal is any scalar-valued function
\begin{equation}
\nu_t = s\!\left(\theta_t^{\mathrm{prop}}, \mathcal{H}_t\right),
\end{equation}
computed at runtime from a proposed parameter update $\theta_t^{\mathrm{prop}}$ and a finite summary of recent training behavior $\mathcal{H}_t$, where $\mathcal{H}_t$ may include previously accepted measurement values, moving averages, or other bounded statistics derived from past iterations.
The innovation signal is required to satisfy:
\begin{enumerate}
    \item \textbf{Externality.} The signal is not directly optimized by the training objective used to compute gradients.
    \item \textbf{Nominal Reference Consistency.} During ordinary non-failure training, the innovation remains within a tolerance band around its recent reference behavior; equivalently, accepted updates produce residuals that are small relative to the threshold used by the controller.
    \item \textbf{Catastrophic Sensitivity.} Destabilizing updates produce deviations in the signal that are significantly larger than those observed during nominal operation.
\end{enumerate}
\end{definition}

\paragraph{Instantiation Used in This Work.}
In our experiments, we instantiate the innovation signal using a fixed validation probe of $|P|=16$ examples:
\begin{equation}
\nu_t = y_P(\theta_t^{\mathrm{prop}}) - \hat{y}_t,
\end{equation}
where $y_P(\cdot)$ denotes the average loss evaluated on the fixed probe set $P$, and $\hat{y}_t$ is an exponentially weighted moving average of past accepted probe measurements. The probe set remains unchanged throughout training, providing a consistent external reference against which optimizer-proposed updates are evaluated. This choice satisfies all admissibility criteria: the probe loss is external to the optimizer, remains stable during nominal training, and responds sharply to destabilizing updates.

\paragraph{Why Training Loss Is Insufficient.}
A natural alternative is to monitor the training loss directly and rollback upon loss spikes. However, training loss is evaluated on the same minibatch used to generate the update and is therefore tightly coupled to minibatch noise. In practice, this coupling leads to frequent false positives and delayed detection of catastrophic failures. By contrast, probe-based innovation leverages information not seen by the optimizer, enabling earlier and more reliable detection with substantially fewer interventions.
\subsection{Control Actions}

When the innovation signal indicates instability, the controller applies a binary runtime decision rule: the proposed update is either accepted and applied as-is, or rejected via rollback to the most recent accepted training state. This design focuses on minimal, decisive interventions that preserve standard optimization dynamics whenever possible.

Figure~\ref{fig:controller_overview} illustrates how the optimizer, measurement signal, and stability controller interact during training. The optimizer computes a proposed parameter update from the current training batch, while the measurement signal evaluates the effect of that proposal before it is accepted into the training trajectory. The stability controller uses this measurement to decide whether the update should be accepted or whether the system should restore the most recent safe parameter and optimizer state.

\subsection{Measurement Signals}

The controller relies on a secondary \emph{measurement signal} to assess the effect of a proposed update. This signal is evaluated at runtime and is not required to coincide with the training loss. Suitable measurement signals include, but are not limited to:
(i) validation probe loss computed on a small, fixed subset of held-out data,
(ii) statistics of the training loss trajectory, such as abrupt increases or deviations from recent trends,
and (iii) measures of gradient agreement, such as cosine similarity between gradients computed on different minibatches.

\begin{figure}[t]
\centering
\resizebox{\linewidth}{!}{%
\begin{tikzpicture}[
box/.style={
draw=black!70,
rounded corners,
minimum width=2.45cm,
minimum height=0.85cm,
text width=2.35cm,
align=center,
fill=gray!10
},
opt/.style={
box,
fill=gray!18
},
ctrl/.style={
box,
fill=blue!15
},
arrow/.style={
->,
thick,
draw=black!70
}
]

% Top row: training flow
\node[box] (batch) {Training Batch};
\node[opt, right=1.15cm of batch] (opt) {Optimizer};
\node[box, right=1.15cm of opt] (prop) {Proposed Update};
\node[box, right=1.15cm of prop] (state) {Model Parameters};

% Bottom row: monitoring and control
\node[box, below=1.15cm of prop] (measure) {Measurement Signal};
\node[ctrl, right=1.15cm of measure] (ctrl) {Stability Controller};

% Main training flow
\draw[arrow] (batch) -- (opt);
\draw[arrow] (opt) -- (prop);
\draw[arrow] (prop) -- node[above, font=\scriptsize]{accept} (state);

% Monitoring path
\draw[arrow] (prop) -- node[right, font=\scriptsize]{evaluate} (measure);
\draw[arrow] (measure) -- node[above, font=\scriptsize]{innovation} (ctrl);

% Recovery action
\draw[arrow, dashed] (ctrl.north) -- node[right, font=\scriptsize]{restore safe state} (state.south);

\end{tikzpicture}
}
\caption{Runtime stability controller as an external supervisory layer. The optimizer proposes parameter updates based on training batches. A secondary measurement signal evaluates the proposed update, and a stability controller decides whether to accept it or trigger rollback by restoring a previously accepted safe state. The controller operates solely on the measurement signal and stored snapshots.}
\label{fig:controller_overview}
\vspace{-0.2cm}
\end{figure}

\section{Algorithm}

We now formalize the runtime stability controller in Algorithm~\ref{alg:stability_controller}. 

Let $\theta_t \in \mathrm{R}^d$ denote the active model parameters at iteration $t$, and let $\omega_t$ denote the active optimizer state. The optimizer state includes all internal variables maintained by the optimizer, such as momentum buffers, adaptive moment estimates, learning-rate schedules, and any other stateful quantities required to continue training. We define the active training state as
\begin{equation}
X_t := (\theta_t,\omega_t).
\end{equation}

At each iteration, the base optimizer induces a tentative transition from the active state $X_t$ to a proposed state
\begin{equation}
X_t^{\mathrm{prop}} := 
(\theta_t^{\mathrm{prop}},\omega_t^{\mathrm{prop}}),
\end{equation}
where $\theta_t^{\mathrm{prop}}$ denotes the proposed parameter vector and $\omega_t^{\mathrm{prop}}$ denotes the corresponding proposed optimizer state. This distinction is necessary for stateful optimizers: an optimizer step changes not only the parameters but also the optimizer memory.

The controller maintains a separate recovery snapshot
\begin{equation}
X_{\mathrm{safe}} := 
(\theta_{\mathrm{safe}},\omega_{\mathrm{safe}}),
\end{equation}
which stores the most recent accepted training state. The distinction between $X_t$ and $X_{\mathrm{safe}}$ is important. The active state $X_t$ is the state from which the optimizer proposes the next transition, whereas $X_{\mathrm{safe}}$ is the state restored when the proposal is rejected. Thus, rollback is defined over the full training state, not merely over the model parameters.

The controller also maintains a scalar reference signal $\hat{y}_t$, computed from previously accepted measurements. Given a measurement function $y(\cdot)$, the innovation associated with a proposed update is
\begin{equation}
\nu_t = y(\theta_t^{\mathrm{prop}}) - \hat{y}_t.
\end{equation}
The proposed transition is accepted only if $\nu_t \leq \epsilon$, where $\epsilon$ is the prescribed innovation tolerance. If $\nu_t > \epsilon$, the proposed transition is rejected and the controller restores the most recent safe state.

\begin{algorithm}[t]
\caption{Runtime Stability Controller with State Recovery}
\label{alg:stability_controller}
\begin{algorithmic}[1]
\STATE \textbf{Input:} Initial parameters $\theta_0$, initial optimizer state $\omega_0$, measurement function $y(\cdot)$, threshold $\epsilon$, smoothing parameter $\alpha$
\STATE $\hat{y}_0 \leftarrow y(\theta_0)$
\STATE $(\theta_{\mathrm{safe}},\omega_{\mathrm{safe}}) \leftarrow (\theta_0,\omega_0)$
\FOR{$t = 0,1,2,\dots$}
    \STATE $(\Delta\theta_t,\omega_t^{\mathrm{prop}}) \leftarrow \operatorname{OptStep}(\theta_t,\omega_t)$
    \STATE $\theta_t^{\mathrm{prop}} \leftarrow \theta_t + \Delta\theta_t$
    \STATE $\nu_t \leftarrow y(\theta_t^{\mathrm{prop}}) - \hat{y}_t$
    \IF{$\nu_t \leq \epsilon$}
        \STATE $\theta_{t+1} \leftarrow \theta_t^{\mathrm{prop}}$
        \STATE $\omega_{t+1} \leftarrow \omega_t^{\mathrm{prop}}$
        \STATE $(\theta_{\mathrm{safe}},\omega_{\mathrm{safe}}) \leftarrow (\theta_{t+1},\omega_{t+1})$
        \STATE $\hat{y}_{t+1} \leftarrow (1-\alpha)\hat{y}_t + \alpha y(\theta_{t+1})$
    \ELSE
        \STATE $(\theta_{t+1},\omega_{t+1}) \leftarrow (\theta_{\mathrm{safe}},\omega_{\mathrm{safe}})$
        \STATE $\hat{y}_{t+1} \leftarrow \hat{y}_t$
    \ENDIF
\ENDFOR
\end{algorithmic}
\end{algorithm}

Algorithm~\ref{alg:stability_controller} implements an innovation-based acceptance rule over the full training state. Lines~1-3 initialize the measurement reference and the recovery snapshot $(\theta_{\mathrm{safe}},\omega_{\mathrm{safe}})$, which stores both model parameters and optimizer state. Lines~5-7 form a proposed transition. The base optimizer is applied to the active state $(\theta_t,\omega_t)$, producing a parameter update $\Delta\theta_t$ and a proposed optimizer state $\omega_t^{\mathrm{prop}}$. The controller then evaluates the proposed parameters through the innovation $\nu_t = y(\theta_t^{\mathrm{prop}})-\hat{y}_t$. Lines~8-12 describe acceptance: if $\nu_t \leq \epsilon$, both the proposed parameters and optimizer state are promoted to the next active state, the recovery snapshot is updated, and the reference signal is refreshed using the accepted measurement. Lines~13-15 describe rejection: if $\nu_t > \epsilon$, the proposed transition is discarded, the full training state is restored from the recovery snapshot, and the reference signal is left unchanged. Thus, rollback is a full-state operation over $(\theta,\omega)$ rather than a parameter-only reset. This is essential for stateful optimizers such as AdamW, where optimizer memory must be restored together with the model parameters.

\section{Theoretical Properties}

In this section, we analyze the theoretical properties of the proposed runtime stability controller. Rather than attempting to establish convergence guarantees for nonconvex optimization, which are generally unattainable for modern deep learning systems, we focus on \emph{runtime safety and recovery properties}. 

\subsection{Bounded Degradation Invariant}
\begin{assumption}[Rollback Capability]
\label{ass:rollback}
The controller maintains a snapshot buffer storing the most recent accepted parameter and optimizer state $(\theta_t, \mathcal{O}_t)$ and can restore this state exactly when a rollback action is triggered.
\end{assumption}

\begin{assumption}[Acceptance Threshold on Innovation]
\label{ass:threshold}
There exists a tolerance $\epsilon > 0$ such that the controller accepts a proposed update at iteration $t$ only if
\begin{equation}
\nu_t \;:=\; y(\theta_t^{\mathrm{prop}}) - \hat{y}_t \;\le\; \epsilon,
\end{equation}
and otherwise triggers rollback by restoring the most recent safe snapshot.
\end{assumption}

\begin{theorem}[Bounded Deviation from the Reference Signal]
\label{thm:bounded}
Under Assumptions~\ref{ass:rollback} and~\ref{ass:threshold}, the sequence of accepted states produced by the controller satisfies, for all $t$,
\begin{equation}
y(\theta_{t+1}) \;\le\; \hat{y}_t + \epsilon.
\end{equation}
\end{theorem}

\begin{proof}
At iteration $t$, the optimizer proposes $\theta_t^{\mathrm{prop}}$ and the controller computes
$\nu_t = y(\theta_t^{\mathrm{prop}}) - \hat{y}_t$.
If $\nu_t \le \epsilon$, the update is accepted and $\theta_{t+1} = \theta_t^{\mathrm{prop}}$, hence
$y(\theta_{t+1}) = y(\theta_t^{\mathrm{prop}}) \le \hat{y}_t + \epsilon$.
If $\nu_t > \epsilon$, rollback is triggered and the controller restores the safe snapshot, yielding $\theta_{t+1} = \theta_t$ (the most recent accepted state), so the accepted state remains unchanged and the inequality holds vacuously because it concerns accepted transitions only.
\end{proof}

Theorem~\ref{thm:bounded} establishes a \emph{runtime safety invariant}: regardless of optimizer behavior, the controller prevents unbounded instantaneous degradation in the measurement signal.
\subsection{Recovery from Destabilizing Updates}

We next formalize the recovery property enforced by the controller when a proposed update is rejected. The relevant object is the full training state, not the parameter vector alone.

\begin{proposition}[Rejection Invariant and Full-State Recovery]
\label{prop:recovery}
Let $X_t=(\theta_t,\omega_t)$ denote the active training state at iteration $t$, where $\theta_t$ is the model parameter vector and $\omega_t$ is the optimizer state. Let $X_t^{\mathrm{prop}}=(\theta_t^{\mathrm{prop}},\omega_t^{\mathrm{prop}})$ 
denote the optimizer-proposed state, and let $X_{\mathrm{safe}}^t=(\theta_{\mathrm{safe}}^t,\omega_{\mathrm{safe}}^t)$ denote the recovery snapshot available to the controller at iteration $t$.

Assume exact rollback semantics: whenever a proposal is rejected, the next active state is restored from the available recovery snapshot. Assume also that the reference signal is updated only after accepted transitions. If the controller rejects the proposal at iteration $t$, i.e. $\nu_t>\epsilon$, then
\begin{equation}
X_{t+1}=X_{\mathrm{safe}}^t
\quad\text{and}\quad
\hat{y}_{t+1}=\hat{y}_t.
\end{equation}
Equivalently,
\begin{equation}
\theta_{t+1}=\theta_{\mathrm{safe}}^t,
\qquad
\omega_{t+1}=\omega_{\mathrm{safe}}^t.
\end{equation}
Consequently, the rejected proposed state $X_t^{\mathrm{prop}}$ is not admitted into either the subsequent training state or the reference dynamics used for future acceptance decisions.
\end{proposition}

\begin{proof}
Since $\nu_t>\epsilon$, the proposed transition is rejected by the controller. By exact rollback semantics, rejection replaces the next active training state with the recovery snapshot available at iteration $t$. Hence $X_{t+1}=X_{\mathrm{safe}}^t$. Taking components yields: 
\begin{equation}
\theta_{t+1}=\theta_{\mathrm{safe}}^t,
\end{equation}
\begin{equation}
    \omega_{t+1}=\omega_{\mathrm{safe}}^t.
\end{equation}

Because the reference signal is updated only after accepted transitions, and the current proposal is rejected, no measurement from the proposed state is incorporated into the reference signal. Therefore,
\begin{equation}
    \hat{y}_{t+1}=\hat{y}_t.
\end{equation}

It follows that the rejected proposal affects neither the next active training state nor the reference signal that governs subsequent acceptance decisions.
\end{proof}

Proposition~\ref{prop:recovery} establishes a rollback invariant for destabilizing updates: rejected proposals are excluded from the accepted training trajectory. Importantly, recovery is defined over the full state $(\theta,\omega)$, including optimizer memory, rather than as a parameter-only reset. This distinction is necessary for stateful optimizers, where retaining optimizer statistics from a rejected update could reintroduce the effects of the destabilizing transition in subsequent steps.

\subsection{Dominance over Unconditional Acceptance}

Finally, we compare the controlled training process against unconditional acceptance of optimizer updates.

\begin{proposition}[Monotone Safety Envelope]
\label{prop:envelope}
Assume $\hat{y}_0 = y(\theta_0)$ and $\hat{y}_{t+1} = (1-\alpha)\hat{y}_t + \alpha y(\theta_{t+1})$ is updated only on accepted steps, with $\alpha \in (0,1)$.
Under Assumptions~\ref{ass:rollback} and~\ref{ass:threshold}, for all $t \ge 0$,
\begin{equation}
y(\theta_{t+1}) \le \max_{k \le t} y(\theta_k) + \epsilon,
\end{equation}
and consequently
\begin{equation}
\max_{k \le t+1} y(\theta_k) \le y(\theta_0) + (t+1)\epsilon.
\end{equation}
\end{proposition}

\begin{proof}
If the proposal at iteration $t$ is rejected, rollback yields $\theta_{t+1}=\theta_t$ by Proposition~\ref{prop:recovery}, hence the bound holds trivially.
If it is accepted, Theorem~\ref{thm:bounded} gives $y(\theta_{t+1}) \le \hat{y}_t + \epsilon$.
Since $\hat{y}_t$ is a convex combination of previously accepted measurements, it satisfies
$\hat{y}_t \le \max_{k \le t} y(\theta_k)$.
Therefore $y(\theta_{t+1}) \le \max_{k \le t} y(\theta_k) + \epsilon$.
The second inequality follows by taking the maximum over $k \le t+1$ and iterating the one-step envelope bound.
\end{proof}

The results above establish that the proposed controller enforces meaningful runtime guarantees independent of optimizer design or loss geometry. Importantly, these guarantees do not rely on convexity, smoothness, or probabilistic modeling assumptions. Instead, they formalize training reliability as a safety property, analogous to invariants used in control and safety-critical systems.

We emphasize that these guarantees apply to the rollback-based controller implemented and evaluated in this work. Extensions to richer action sets, such as update attenuation or regime switching, are conceptually compatible with the framework but are left for future work.

\section{Scale and Complexity Analysis}

A primary concern for runtime monitoring is the trade-off between stability and throughput. We analyze the overhead of the stability layer below.
\subsection{Computational Overhead}

The controller introduces additional computation through evaluation of the measurement signal. Let $C_{\mathrm{step}}$ denote the cost of a standard optimizer step, including the forward and backward passes on the training batch, and let $C_{\mathrm{probe}}$ denote the cost of evaluating the probe measurement. If the probe is evaluated every $q$ training steps, the average relative compute overhead is
\begin{equation}
\gamma_q =
\frac{1}{q}\cdot
\frac{C_{\mathrm{probe}}}{C_{\mathrm{step}}}.
\end{equation}
Thus, the overhead is controlled by two implementation choices: the size of the probe set and the measurement frequency. In the common case where the probe requires only a forward pass on a small fixed subset, $C_{\mathrm{probe}}$ is substantially smaller than the cost of a full training step. Moreover, intermittent measurement can reduce the average overhead without changing the underlying acceptance rule.

\section{Experiments}
We evaluate the proposed runtime stability controller under controlled destabilization scenarios designed to isolate failure detection, selective intervention, and recovery behavior. The experiments are diagnostic rather than benchmark-driven, and are designed to illustrate how the controller behaves under rare but severe training failures. All results are aggregated over $N=20$ independent random seeds.

\subsection{Catastrophic Recovery in Vision Models}

Figure~\ref{fig:resnet_recovery} shows probe loss trajectories for a ResNet-18 trained on CIFAR-10 under a multi-step catastrophic perturbation induced by gradient amplification. The baseline exhibits a pronounced loss spike followed by slow and highly variable recovery across seeds. In contrast, the controlled runs consistently limit peak degradation and return to a stable regime more rapidly, with substantially reduced variance.

\subsection{Innovation-Based Detection and Selective Intervention}

The innovation signal $\nu_t$ exhibits a sharp and localized deviation during the failure window (Figure~\ref{fig:innovation_signal}). Outside this window, the signal remains tightly bounded, indicating stable behavior under nominal training dynamics. This contrast demonstrates that the signal exhibits a sharp transition, crossing the safety threshold $\epsilon$. At step $t=120$, the signal exceeds the predefined safety threshold $\epsilon$. Because the controller monitors a probe set decoupled from the optimizer's gradients, it provides earlier and more reliable detection of destabilizing behavior than loss-based triggers, without introducing false positives during nominal training.

\subsection{Internal Stability via Parameter Norms}

Figure~\ref{fig:weight_norms} examines the evolution of the $\ell_2$ norm of model parameters during training. Following destabilization, baseline training transitions into a higher-energy parameter regime and remains there for the remainder of training. In contrast, controlled runs return to parameter norms consistent with nominal training behavior.

This suggests that the controller constrains internal model dynamics rather than merely correcting observable loss spikes, preventing irreversible drift into regions of parameter space often correlated with brittle behavior in practice.

\subsection{Generalization to Transformer Models}

To assess architectural generality, we repeat the recovery experiment using a character-level Transformer model. As shown in Figure~\ref{fig:transformer_recovery}, the controller again reduces peak degradation and accelerates recovery following the injected perturbation, while the baseline exhibits a prolonged elevated-loss regime. This result indicates that the proposed mechanism generalizes beyond convolutional architectures and is not tied to specific inductive biases.

\subsection{Intervention Sparsity}

Across all experiments, the controller intervenes sparsely. Rollbacks are typically triggered within a small number of iterations following perturbation onset, and rarely outside the injected failure window. In contrast, a naive loss-based rollback strategy exhibits delayed detection and frequent false positives due to minibatch noise. 

\begin{figure}[!t]
  \centering
  \begin{minipage}[t]{0.485\linewidth}
    \centering
    \includegraphics[width=\linewidth]{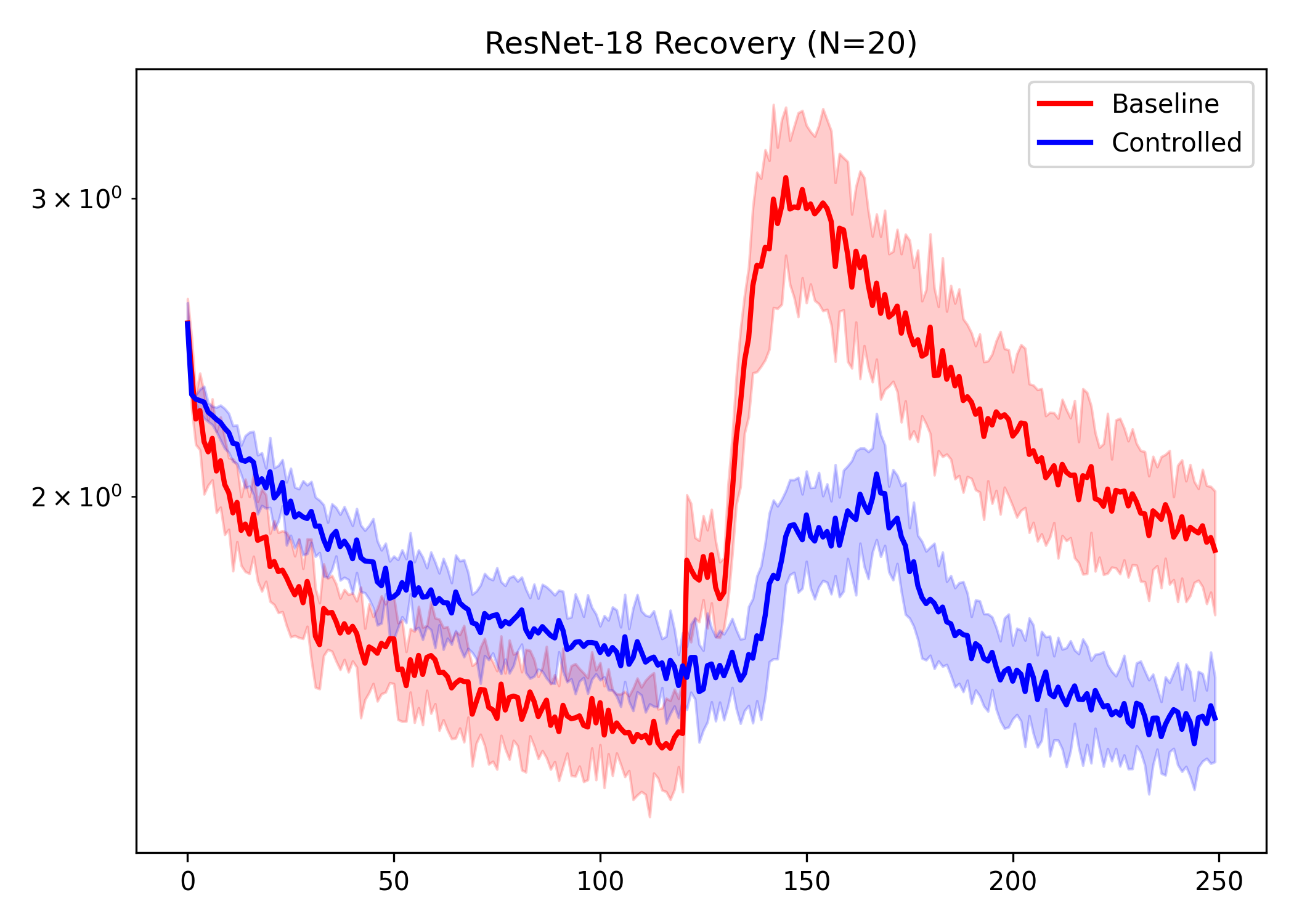}
    \caption{Probe loss recovery for ResNet-18 on CIFAR-10 under a multi-step catastrophic perturbation. Curves show mean $\pm$ standard deviation over $N=20$ seeds.}
    \label{fig:resnet_recovery}
  \end{minipage}\hfill
  \begin{minipage}[t]{0.485\linewidth}
    \centering
    \includegraphics[width=\linewidth]{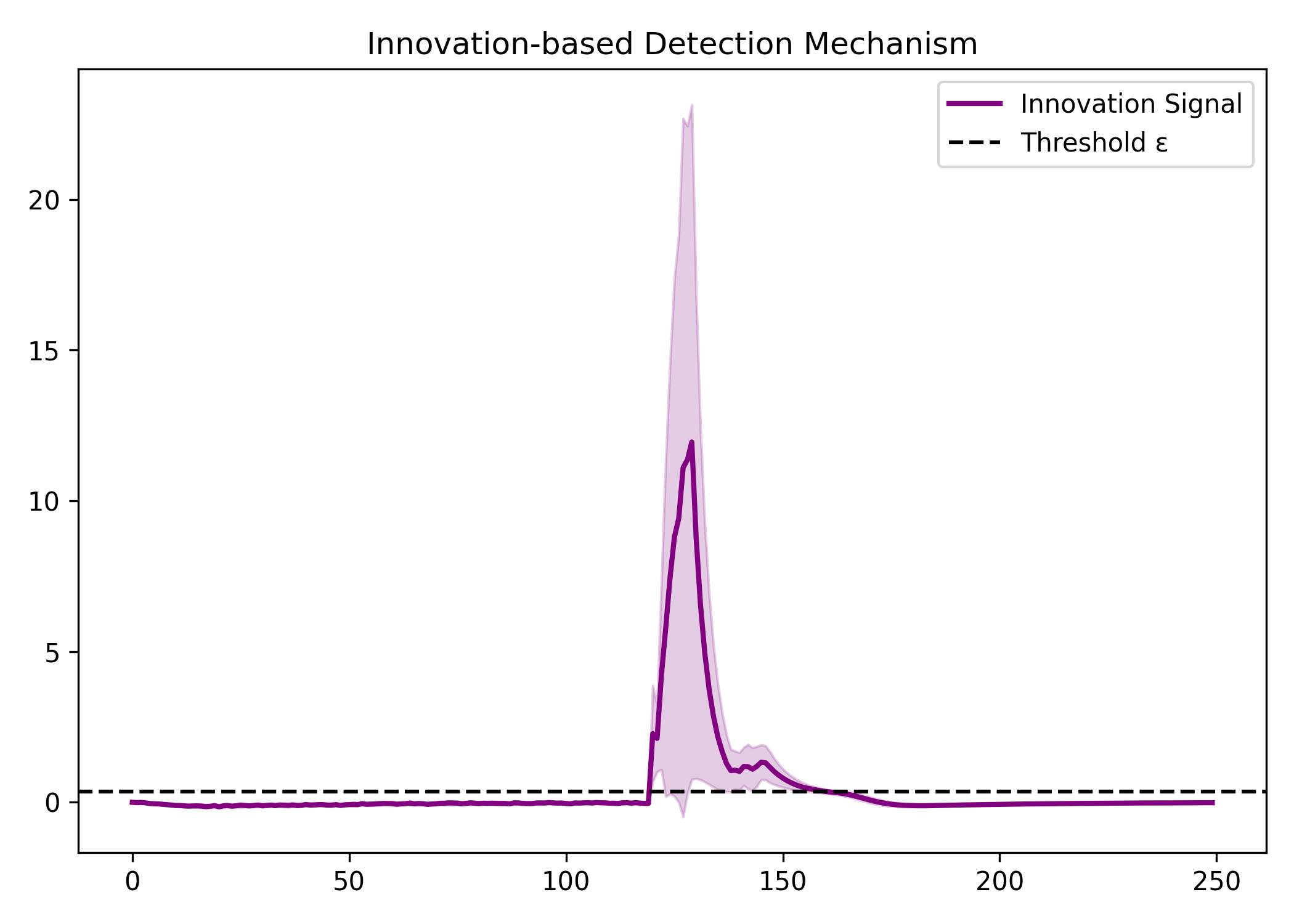}
    \caption{Innovation signal $\nu_t$ during the injected destabilization window. The localized deviation enables reliable detection without sensitivity to minibatch noise.}
    \label{fig:innovation_signal}
  \end{minipage}

  \vspace{0.45em}

  \begin{minipage}[t]{0.485\linewidth}
    \centering
    \includegraphics[width=\linewidth]{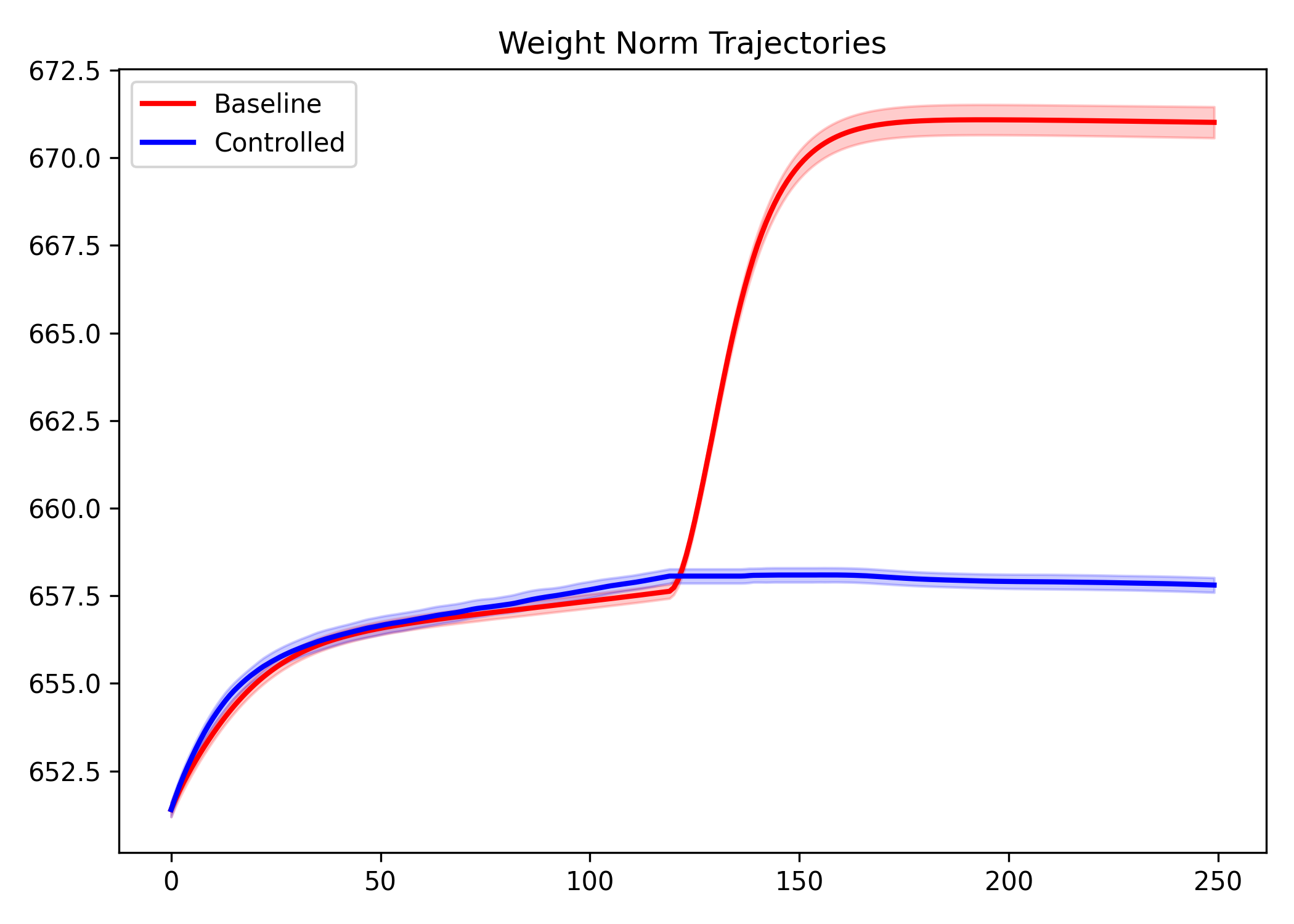}
    \caption{Evolution of parameter $\ell_2$ norms during training. The controller prevents persistent drift into high-norm regimes after destabilization.}
    \label{fig:weight_norms}
  \end{minipage}\hfill
  \begin{minipage}[t]{0.485\linewidth}
    \centering
    \includegraphics[width=\linewidth]{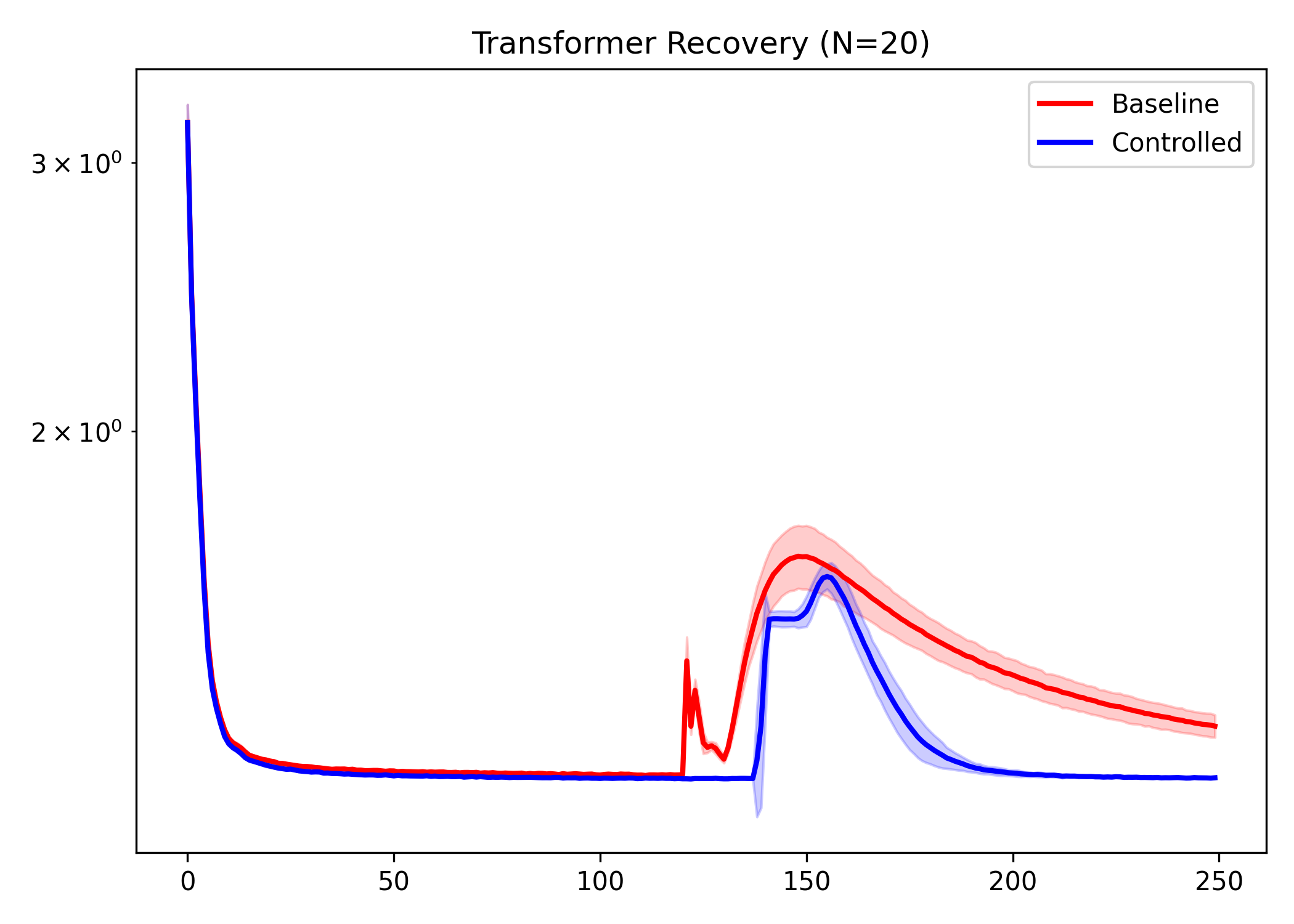}
    \caption{Probe loss recovery for a character-level Transformer under catastrophic perturbation. Results are averaged over $N=20$ seeds.}
    \label{fig:transformer_recovery}
  \end{minipage}
\end{figure}
\section{Discussion and Limitations}

The empirical results support the view of runtime stability as an acceptance problem rather than only an optimizer-design problem. In the ResNet-18 experiment, Figure~\ref{fig:resnet_recovery} shows that the baseline trajectory suffers a large probe-loss excursion after the injected perturbation, followed by slow and high-variance recovery. In contrast, the controlled trajectory limits the peak degradation and returns more rapidly to the pre-perturbation regime. This behavior is consistent with the intended role of the controller: it does not attempt to make every optimizer step safer in advance, but instead prevents destabilizing proposals from being incorporated into the accepted training trajectory.

The innovation signal provides the mechanism behind this behavior. As shown in Figure~\ref{fig:innovation_signal}, the signal remains close to its nominal range during ordinary training and exhibits a sharp, localized deviation during the destabilization window. This separation is important because it indicates that the controller is not merely reacting to noisy fluctuations in the training loss. Rather, the fixed probe measurement produces a stable external reference against which proposed updates can be evaluated. The observed sparsity of interventions further suggests that the controller is selective: rollback is triggered near the failure event and is not used as a persistent correction mechanism during nominal optimization.

Figure~\ref{fig:weight_norms} provides complementary evidence that the controller affects the internal training dynamics rather than only the observed probe loss. After destabilization, the baseline enters a higher-norm parameter regime and remains there, while the controlled run returns to a trajectory consistent with nominal training. This supports the interpretation that rollback prevents persistent drift in parameter space after a catastrophic update. The Transformer experiment in Figure~\ref{fig:transformer_recovery} further indicates that the same mechanism is not specific to convolutional architectures: the controller again reduces peak degradation and accelerates recovery in a sequence-modeling setting.

Taken together, the results show that a lightweight supervisory layer can provide recovery behavior that is absent from standard optimizer-only training pipelines. The controller preserves ordinary optimization during stable regimes, but introduces a state-restoration mechanism when the measured effect of a proposed update becomes inconsistent with recent accepted behavior. This is the central practical distinction from preventive methods such as clipping, adaptive learning rates, or trust-region-style constraints: those methods seek to reduce the likelihood of harmful updates, whereas the proposed controller supplies an explicit runtime mechanism for detecting and excluding them after they are proposed.

The main limitation of the present study is that the experiments are diagnostic rather than exhaustive: they use controlled destabilization, a fixed probe-based measurement signal, and a binary accept-or-rollback action set. Future work should evaluate broader failure modes, adaptive thresholding, alternative measurement signals, and distributed training settings. These extensions do not change the core mechanism studied here, but they are necessary for characterizing the controller under more diverse large-scale training conditions.

\section{Conclusion}

We presented a runtime stability controller for neural network training that emphasizes reliability and recovery rather than convergence guarantees. By framing training as a controlled stochastic process and introducing a mechanism that evaluates optimizer-proposed updates using secondary measurement signals, the proposed approach enables detection and mitigation of destabilizing events that can irreversibly degrade training. 

Unlike existing optimization methods that embed preventive mechanisms within the update rule, the proposed controller operates as an external layer that governs the acceptance and execution of updates at runtime. This design preserves compatibility with standard optimizers while enabling exact recovery through rollback when instability is detected. We provided safety-style theoretical guarantees that formalize bounded degradation and one-step recovery properties, and demonstrated empirically that the controller reduces peak degradation, stabilizes internal model dynamics, and improves reliability across both convolutional and Transformer-based architectures. Our work provides a blueprint for \textbf{self-healing training pipelines}, shifting the paradigm from manual hyperparameter tuning to automated, control-theoretic runtime stability.

\end{document}